\documentclass[conference]{IEEEtran}
\IEEEoverridecommandlockouts
\usepackage{cite}
\usepackage{amsmath,amssymb,amsfonts}
\usepackage{algorithmic}
\usepackage{algorithm}
\usepackage{listings}
\usepackage{breakurl}
\usepackage{url}
\usepackage{graphicx}
\usepackage{textcomp}
\usepackage{xcolor}
\def\BibTeX{{\rm B\kern-.05em{\sc i\kern-.025em b}\kern-.08em
    T\kern-.1667em\lower.7ex\hbox{E}\kern-.125emX}}
    
\usepackage{amsthm}
\usepackage{tikz}
\usepackage{pgfplots}
\usepackage{cleveref}
\usepackage{enumitem}

\newcommand{\fpname}{floating point}
\newcommand{\fpexname}{IEEE 754-1985}
\newcommand{\tkname}{two's complement}
\newcommand{\mdname}{FLInt}
\newcommand{\mdnames}{FLInts}

\newcommand{\kuanmdname}{cache-aware grouping and swapping}
\newcommand{\kuanmdnameabbr}{CAGS}

\newtheorem{lemma}{Lemma}
\newtheorem{theorem}{Theorem}
\newtheorem{corollary}{Corollary}
\newtheorem{definition}{Definition}

\begin{document}

\title{\textbf{FLInt:} Exploiting \underline{Fl}oating Point Enabled \underline{Int}eger Arithmetic for Efficient Random Forest Inference \vspace{-.5cm}}

\author{Christian Hakert (TU Dortmund University), Kuan-Hsun Chen (University of Twente), \\Jian-Jia Chen (TU Dortmund University)\vspace{-.5cm}}

\maketitle
\thispagestyle{plain}

\begin{abstract}
In many machine learning applications, e.g., tree-based ensembles, floating point numbers are extensively utilized due to their expressiveness.  
Nowadays performing data analysis on embedded devices from dynamic data masses becomes available, but such systems often lack hardware capabilities to process floating point numbers, introducing large overheads for their processing.
Even if such hardware is present in general computing systems,
using integer operations instead of floating point operations promises to reduce operation overheads and improve the performance.

In this paper, we provide \mdname, a full precision floating point comparison for random forests, by only using integer and logic operations. To ensure the same functionality preserves, we formally prove the correctness of this comparison. 
Since random forests only require comparison of floating point numbers during inference, we implement \mdname~in low level realizations and therefore eliminate the need for floating point hardware entirely, by keeping the model accuracy unchanged. 
The usage of \mdname~basically boils down to a one-by-one replacement of conditions: For instance, a comparison statement in C:
\texttt{if(pX[3]<=(float)10.074347)} becomes
\texttt{if((*(((int*)(pX))+3))<=((int)(0x41213087)))}. Experimental evaluation on X86 and ARMv8 desktop and server class systems shows that the execution time can be reduced by up to $\approx 30\%$ with our novel approach.
\end{abstract}


% \clearpage
\section{Introduction}

Random forests are famous machine learning models, especially for scenarios under resource limitations. While training of the models usually can be done on powerful hardware,
inference has to be highly resource efficient in order to exploit maximal performance on the execution platform.
Although random forests in their structure can be made highly resource efficient, the processed data is given and hardly can be optimized. If it is required to use floating point numbers for the classification, systems need to be equipped with a hardware floating point unit or consume more energy and time for the use of software floating points \cite{cowlishaw2003decimal}. For small embedded systems, floating point units are commonly not integrated into the hardware. Even if a floating point unit is present, it may be desired to not activate it in order to save energy. Even out of the scope of embedded systems, the usage of floating point units usually introduces a certain overhead, at least in terms of execution time. One trivial approach would be to round all floating point numbers to integers, which potentially induces a loss in accuracy.
A certain floating point arithmetic is still required if incoming data still is encoded as floating points.

In this work, we provide a new alternative. The studied \textbf{problem} is to \textit{compute correct floating point arithmetic in random forests without the need for hardware floating point support}. By only using standard integer and logic operations, we 1) enable floating-point-based random forests on devices without floating point hardware and 2) eliminate the overheads to use the floating point unit. We answer the question: \textit{How floating point comparisons can be correctly computed based on integer and logic operations?} We \textbf{solve} this problem by specifically investigating the binary floating point format \cite{30711} and consider the binary ordering in relation to \tkname~signed integer interpretation \cite{10.5555/3153875}. While we formally prove that positive floating point numbers are order preserving,
the handling of negative numbers and a few special cases requires dedicated handling. Such a handling is integrated into a single operator, which we call \textbf{\mdname}. 
To the best of our knowledge, this work is the first to formally prove the correctness of an integer and logic based floating point operator and integrate it into the implementation of random forests.

We further consider an efficient implementation of the \mdname~operator in existing random forest implementations \cite{chen2022efficient} by processing the aforementioned handling during the implementation time to minimize the introduced overhead.
In order to evaluate the effectiveness of \mdname, we conduct experimental evaluation on the implementation of if-else tree based random forest ensembles. The results show that we can reduce the execution time by up to $\approx 30\%$ with the use of \mdname~on desktop and server systems.

\noindent\textbf{Our novel contributions:}
\begin{itemize}
    \item \mdname: A \tkname~and logic operation based comparison operator for floating point numbers, where we formally prove the correctness (see Section~\ref{sec:flint}).
    \item An efficient implementation of \mdname~in random forests with if-else tree implementations, where we resolve special case handling offline during the implementation time (see Section~\ref{sec_dt}).
    \item Experimental evaluation on X86 and ARMv8 server and desktop class systems to study the reduction of overall execution time when using \mdname~instead of floating points (see Section~\ref{sec_eval}).
\end{itemize}

\section{Related Work}
\label{sec_relwork}
Considering floating point arithmetic and the relation to \tkname~integer arithmetic, several starting points can be found in the literature. It is reported that some CPUs internally use the same hardware unit for floating point and integer comparisons with a few additions for the floating point computation \cite{bramley}. Furthermore, there are explicit considerations about the binary floating point format regarding accuracy and efficient programming \cite{demmel1995correctness, blinn1997floating}. However, comparing floating point numbers from the application level with the help of integer operations and providing guarantees for the correctness, to the best of our knowledge, is not considered. Furthermore, the integration of such a method into random forest execution is also not studied yet.

Efficient execution of random forests and its basic unit, decision trees, has been widely studied in the literature. The relevant techniques can be informally divided into two families of approaches: \textit{algorithmic refinements} and \textit{architectural optimizations}. The former targets to improve the execution algorithm of decision trees itself, e.g., by applying different representations. Kim et al.~propose parallelization in the form of vectorization for decision trees to favor the usage of Intel CPUs~\cite{kim/etal/2010}. Based on vectorization, QuickScorer performs a interleaved traversal of the trees by using logical bitwise operations for gradient boosted trees~\cite{Lucchese2016}, which is further optimized for batch-processing~\cite{10.1145/2911451.2914758}. Hummingbird is proposed to compile decision trees into a small set of tensor operations and batch tensors for each tree together for tree traversal~\cite{10.5555/3488766.3488817}.

On the other hand, architectural optimizations focus on the utilization of hardware resources.
Asadi~et~al. first introduce the concept of native trees and if-else trees as a low-level implementation \cite{Asadi/etal/2014}. These implementations are later picked up by Buschj\"ager et~al.~\cite{8594826}, where empirical probabilities for single branches within a decision tree are collected on the training data set and used to layout the memory representation of decision trees in order to benefit cache prefetching and protect from preemption. Chen~et~al. in addition utilize the GNU binary utilities to derive exact binary sizes to further optimize the memory layout  \cite{chen2022efficient}. These approaches, however, do not consider the handling of datatypes, which are defined by the training step. Our work can be classified into this category, where we develop a new operator to compare floating point numbers with only integer and logic operations. In addition, we show that the integration of this operator to existing optimization is applicable, bringing up further improvement.

% \clearpage
\section{Providing Correct Floating Point Comparisons with Integer and Logic Arithmetic}
\label{sec:flint}
Floating point arithmetic includes several operations, which are required to process floating point numbers. In this paper, we focus on the comparison only (i.e.~$\geq$), since this is the only operation needed during random forest inference. To eliminate the use of hardware floating point support or software float implementations, we realize a comparison operation by only using signed integer arithmetic and logic operations. In this section, we present the binary layout of floating point numbers and \tkname~numbers, and construct the comparison operator between them. We formally prove every step and conclude the correctness of our final operator.

\subsection{Binary Floating Point and Signed Integer Format}

In order to illustrate the relation between the binary representation of floating point numbers and signed integer numbers, we lay out the state-of-the art formats in the following. Almost all computer systems nowadays use \tkname~\cite{10.5555/3153875} for signed integer numbers and \fpexname~\cite{30711} for single or double precision floating point numbers. Both formats support, among others, 32 bit and 64 bit types. 
For the rest of this section, we define the \fpname~and \tkname~format independent of the total bit length. 32 and 64 bit numbers in \tkname~and in \fpexname~then can be interpreted as an instance of the defined format. Later in this paper (i.e., Section~\ref{sec_dt}), 
we discuss real implementation on common hardware, where we use 32 and 64 bit \tkname~and \fpexname~numbers.

Both formats, \tkname~and \fpname, define an interpretation of a fixed length bit vector. Thus, an arbitrary bit vector can be either interpreted, among other options, as a signed integer, an unsigned integer or a floating point. Furthermore, the binary representation of a floating point number can be interpreted as a signed integer and vice versa.
\begin{definition}
Let $B\in\{0,1\}^{k}$ be a $k$ bit wide vector, then these bits can be interpreted either as a floating point or as signed integer number in \tkname. We denote $FP: \{0,1\}^{k} \to \mathbb{Q}$ as the floating point interpretation $FP(B)$ of $B$, $SI: \{0,1\}^{k} \to \mathbb{Z}$ as the signed integer (\tkname) interpretation $SI(B)$ of $B$ and $UI: \{0,1\}^{k} \to \mathbb{N}$ as the unsigned integer interpretation $UI(B)$ of $B$.
\end{definition}

For the interpretation of signed and unsigned integers, every bit within the bit vector is assigned a fixed value ($2^i$). If the bit is set to 1, the value of the position is counted, otherwise it is ignored. Negative numbers for the signed \tkname~format always have the most significant bit (MSB) set to $1$. The signed value is then interpreted by assigning a negative value to the MSB and interpreting the other bits similar to the unsigned format.
\begin{definition}
\label{tkdef}
The \tkname~interpretation (signed integer) of a bit vector $B=(b_{k-1},...,b_0)$ is defined as
\begin{equation}
    SI(B)=
    \left\{\begin{array}{cc}
         \sum\limits_{i=0}^{k-1} b_i\cdot 2^i& b_{k-1}=0  \\
         -2^{k-1}+\sum\limits_{i=0}^{k-2} b_i\cdot 2^i& b_{k-1}=1 
    \end{array}\right.
    \label{tksieq}
\end{equation}where the unsigned integer interpretation is the same as the \tkname~for positive numbers:
\begin{equation}
    UI(B)=
         \sum\limits_{i=0}^{k-1} b_i\cdot 2^i
\end{equation}
\end{definition}
One key advantage of this format is that the binary ordering (i.e.~the interpretation as an unsigned integer) is the same for negative and positive numbers. In addition, the border between positive and negative numbers allows for unchanged arithmetic, when the overflow bit is ignored. In detail, the representation of $-1$ in \tkname~is $(1,1,1,...)$ and the representation of $0$ is $(0,0,0,...)$. When adding $+1$ to the representation of $-1$ in unsigned arithmetic, all bits switch to $0$ and an overflowing $1$ goes to position $k$. Since the overflow can be safely ignored (which is intended in this case \cite{10.5555/3153875}), the computation is correct.

The \fpname~format

differs from the binary representation of integers. In this format, the binary representation is interpreted as three components: 1) a sign bit at the position of the most significant bit ($k-1$), 2) a biased exponent of $j$ bit length where the bias is $2^{j-1}-1$, and 3) a mantissa, filling the remaining bits, which is interpreted with an implicit $1$.
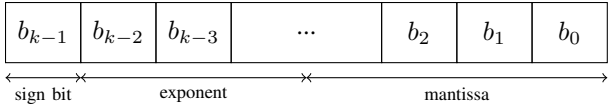
\begin{figure}
    \centering
    \begin{tikzpicture}
    \node[draw, minimum width=1cm, minimum height=.8cm] () at (0,0) {$b_{k-1}$};
    \node[draw, minimum width=1cm, minimum height=.8cm] () at (1,0) {$b_{k-2}$};
    \node[draw, minimum width=1cm, minimum height=.8cm] () at (2,0) {$b_{k-3}$};
    
    \node[draw, minimum width=2cm, minimum height=.8cm] () at (3.5,0) {...};
    
    \node[draw, minimum width=1cm, minimum height=.8cm] () at (5,0) {$b_{2}$};
    \node[draw, minimum width=1cm, minimum height=.8cm] () at (6,0) {$b_{1}$};
    \node[draw, minimum width=1cm, minimum height=.8cm] () at (7,0) {$b_{0}$};
    
    \draw[<->] (-.5,-.6) -- node[below] {\scriptsize sign bit} (.5,-.6);
    \draw[<->] (.5,-.6) -- node[below] {\scriptsize exponent} (3.5,-.6);
    \draw[<->] (3.5,-.6) -- node[below] {\scriptsize mantissa} (7.5,-.6);
    \end{tikzpicture}
    \vspace{-.4cm}
    \caption{Illustration of the binary \fpname~representation}
    \vspace{-.5cm}
    \label{fig:fpformat}
\end{figure}
\Cref{fig:fpformat} illustrates the layout of these three components within the bit vector. The interpretation of the \fpname~format differs in large parts from the interpretation of the \tkname. While the sign bit is interpreted as factor ($\times -1$ or $\times 1$), the mantissa is interpreted as a decimal number between $1$ and $2$, which is scaled by the exponent.
\begin{definition}
\label{fpdef}
The \fpname~interpretation of a bit vector $B=(b_{k-1},...,b_0)=(s, e_{j-1}, ..., e_0, m_{x-1}, ..., m_0)$ for $j$ bit exponent and $x$ bit mantissa is defined as
\begin{equation}
\begin{split}
    FP(B)=(-1)^{s}\cdot 2^{UI(e_{j-1}, ..., e_0)-bias} \\ \cdot (1+\sum\limits_{i=0}^{x-1}m_{i}\cdot2^{-x+i})
\end{split}
\end{equation}
where the bias for the interpretation of the exponent is $bias=2^{j-1}-1$. Please note that the commonly used \fpexname~format is exactly an instance of this format for $j=8, x=23$ (single precision) and $j=11,x=52$ (double precision) \cite{30711}.
\end{definition}

In addition to the normal interpretation of numbers (\Cref{fpdef}), the \fpname~format includes a few exceptional cases for special numbers. The special encoding for positive and negative infinity and the encoding for \emph{not a number} is not further discussed in this paper, since the usage of these numbers does not occur in random forests. If positive or negative infinity should occur anyway, they are encoded as the smallest and largest representable number and thus make no difference for comparison. 

Since the normal interpretation cannot encode a $0$ (due to the implicit $1$ added to the mantissa), the special encoding for the representation of $0.0$ is all bits set to $0$. In addition, the format allows also the encoding of $-0.0$, when the sign bit is set to $1$ and all other bits are set to $0$. In this paper, we assume that $-0.0 < 0.0$, which differs from the definition $-0.0 = 0.0$ of the \fpexname~standard\footnote{$-0$ can be the result of rounding a not representable negative number. Extending our method to handle $-0.0$ equals to $0.0$ is quite straightforward by including one additional scenario during code generation.}. As the mantissa is always interpreted as a number between $1$ and $2$, the smallest representable absolute value would be limited to $2^{-bias}$. To extend this, the \fpname~format includes a \emph{denormalized} format, which is indicated by an exponent of all $0$s. In this format, the exponent is interpreted as $-bias+1$ and the mantissa is interpreted without implicit $1$ (i.e. as a number between $0$ and $1$). This essentially makes the representation of $0.0$ also a valid denormalized number.

\subsection{Ordering Between Floating Points and Signed Integers}

Now we show that the \fpname~format (when the bit vector is interpreted as \tkname) preserves the order of numbers for positive numbers and inverses the order for negative numbers. This is also illustrated in \Cref{fig:decorationfloatspace}, where the signed integer values (respectively, corresponding floating point values) of all combination of 32 bit vectors $B$ are plotted on the $x$ axis (respectively, $y$ axis).

As the intention of this paper is to evaluate the $\geq$ relation of \fpname~numbers in \tkname~arithmetic, we have to consider the equality of numbers first.
\begin{lemma}
\label{lemma:equal}

Given two arbitrary bit vectors $X,Y\in \{0,1\}^{k}$, then the \fpname~interpretation is the same for both numbers, if and only if

also the signed integer representation in \tkname~and the bit vector itself is the same.
\begin{equation}
    FP(X)=FP(Y) \Leftrightarrow X = Y \Leftrightarrow SI(X) = SI(Y)
\end{equation}
\end{lemma}
\begin{proof}
Both formats, \fpname~and \tkname, are bijective for the mapping of the bit vector to a number\footnote{Our definition of the \fpname~format implies that $-0 \neq +0$, which ensures bijectivity. To accommodate for the definition of $-0 = +0$ (as in \fpexname), this case would need to be excluded here and added as a case distinction, presented at the end of  \Cref{sec:C-impl}}. The counted weight for the single bits is a power of $2$ in \fpname~and in \tkname. Hence, the weight of one bit cannot be constructed as a sum of other bits. Furthermore, numbers with a positive sign bit are always positive in both formats, numbers with a negative sign bit are always negative in both formats. Therefore, the bit vector of $X$ and $Y$ must be the same in both formats.
\end{proof}

As already explained in the beginning of this section, the interpretation of signed integer numbers uses the same binary ordering as the interpretation of unsigned integer numbers for both, positive and negative numbers. The floating point interpretation, in contrast, uses the same encoding of the exponent and mantissa for both, positive and negative numbers, and only distinguishes them by the sign bit. Therefore, the sign bit of the \fpname~format can be ignored in order to obtain the absolute value of a floating point number:
\begin{definition}
\label{fpmagndef}
Given a bit vector $B=(b_{k-1},...,b_0)=(s, e_{j-1}, ..., e_0, m_{x-1}, ..., m_0)$, the absolute value of the \fpname~interpretation for $j$ bit exponent and $x$ bit mantissa is defined as
\begin{equation}
    |FP(B)|=2^{UI(e_{j-1}, ..., e_0)-bias} \cdot (1+\sum\limits_{i=0}^{x-1}m_{i}\cdot2^{-x+i})
    \label{fpeq}
\end{equation}
\end{definition}

Since the \tkname~interpretation is order preserving for negative and positive numbers, the absolute value of the \fpname~interpretation follows the same order:

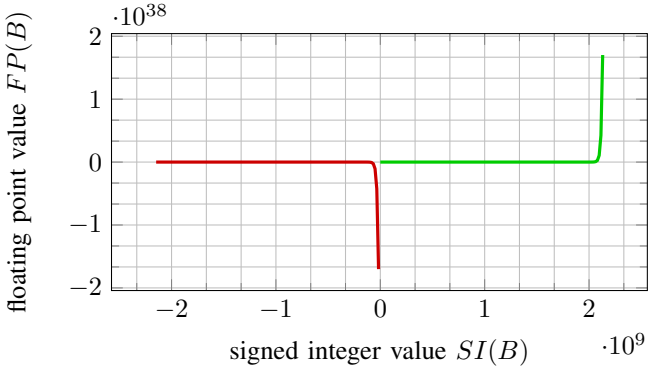
\begin{figure}
    \centering
    \begin{tikzpicture}
        \begin{axis}[grid=both, xlabel=signed integer value $SI(B)$, ylabel=floating point value $FP(B)$, height=5cm, width=.48\textwidth, minor tick num=2]
            \addplot[color=green!80!black, very thick] table [x index=0, y index=1, col sep=comma] {decoration/p.csv};
            \addplot[color=red!80!black, very thick] table [x index=0, y index=1, col sep=comma] {decoration/n.csv};
        \end{axis}
    \end{tikzpicture}
    \vspace{-.4cm}
    \caption{Illustration of signed integer (x axis) and floating point (y axis) space for all combination of 32 bit vectors $B$}
    \vspace{-.5cm}
    \label{fig:decorationfloatspace}
\end{figure}
\begin{lemma}
\label{lemma:monotonic}

Given two arbitrary bit vectors $X,Y\in \{0,1\}^{k}$ with the same sign bit $x_{k-1}=y_{k-1}$, the interpretation as the absolute value of the \fpname~and as signed integer numbers is strictly monotonically increasing:
\begin{equation}
    |FP(X)|>|FP(Y)| \Leftrightarrow SI(X) > SI(Y)
\end{equation}
\end{lemma}
\begin{proof}
Independent of the interpretation as \tkname~or as \fpname, the bit vectors can be divided into the three sections of sign bit, exponent and mantissa (\Cref{fpdef}). Since the sign bits are the same, the numbers have to differ in either the exponent and the mantissa bits according to \Cref{lemma:equal}. Consequently, we have to distinguish two cases:
\begin{itemize}[leftmargin=*]
    \item \textbf{Case 1}: Both numbers have the same exponent: $\Rightarrow$: Then both numbers are either in normal or in denormalized \fpname~format, thus the mantissa is either interpreted with implicit one or without for both numbers in \fpname~interpretation. Thus, the term $\sum\limits_{i=0}^{x-1}m_{i}\cdot2^{-x+i}$ from \Cref{fpmagndef} must evaluate to a larger number for $Y$ as for $X$. In \tkname~the bits of the mantissa are evaluated the same, just with another weight. From \Cref{tkdef}, the mantissa bits contribute with the term $\sum\limits_{i=0}^{x-1}m_{i}\cdot2^{i}=2^{x}\cdot\sum\limits_{i=0}^{x-1}m_{i}\cdot2^{i-x}$ to the total sum, thus with the same value as to the \fpname~format, weighted with a constant factor. Hence, this part of the sum also has to evaluate to a larger number in \tkname. Since the exponents are the same and the sign bit is the same, the remaining part of the sum in \Cref{tksieq} is the same for $X$ and $Y$. Consequently, $X$ evaluates to a larger number in \tkname~as $Y$. $\Leftarrow$: For $X$ to be larger than $Y$ in \tkname~while the exponent bits and the sign bit are the same, the part of the sum of the mantissa bits must evaluate to a larger value for $X$, which also increases the interpretation of the mantissa in \fpname~interpretation, since the bits contribute with another constant weight, as discussed before.
    \item \textbf{Case 2}: The two numbers have different exponent bits: $\Rightarrow$: $X$ must have a larger exponent than $Y$ since the interpretation of the mantissa $m$ ranges between $1 \leq m < 2$. If the exponent of $X$ would be smaller than the exponent of $Y$, the factor, the exponent contributes to \Cref{fpeq} would be at least smaller by a factor of 2. This could only be compensated by the mantissa, if $Y$ would be in denormalized format, which leads to a contradiction, since the denormalized format is encoded by the smallest possible exponent, hence $X$ cannot have a smaller exponent. Consequently, for the exponent of $X$ to be larger, $UI(e_{j-1}, ..., e_0)$ must evaluate to a larger number for $X$. Thus, the same bits must evaluate to a larger number also in \Cref{tksieq}. According to \Cref{tksieq}, the total contribution of the mantissa bits is smaller than any contribution of an exponent bit (as the weight for every higher significant bit is larger than the weight of all lower significant bits summed). Thus, $X$ must also evaluate to a larger number than $Y$ in \tkname. $\Leftarrow$: If the exponent bits are different, the part of the sum for the exponent bits in \Cref{tksieq} must be larger, since the mantissa bits cannot compensate a smaller sum due to their lower total weight. Then, the interpreted exponent in \fpname~must be larger as well, which cannot be compensated by the mantissa, as explained before.
\end{itemize}
Since the representation of $0$ is covered by the denormalized format in \fpname, all cases are considered.
\end{proof}

Next, we have to distinguish the cases for comparing two positive and two negative numbers.
\begin{lemma}
\label{lemma:positivemonotonic}
Given two arbitrary bit vectors $X,Y\in \{0,1\}^{k}$ with \textbf{positive} sign bit $x_{k-1}=y_{k-1}=0$, the interpretation as \fpname~and as signed integer numbers is strictly monotonically increasing:
\begin{equation}
    FP(X)>FP(Y) \Leftrightarrow SI(X) > SI(Y)
\end{equation}
\end{lemma}
\begin{proof}
    When the sign bit of both numbers is $0$, the term $(-1)^s$ in \Cref{fpdef} evaluates to $1$ and has no contribution. Then, the interpretation of the \fpname~number is exactly the same as in \Cref{fpmagndef}. Thus \Cref{lemma:monotonic} holds.
\end{proof}

\begin{lemma}
\label{lemma:negativemonotonic}
Given two arbitrary bit vectors $X,Y\in \{0,1\}^{k}$ with \textbf{negative} sign bit $x_{k-1}=y_{k-1}=1$, the interpretation as \fpname~and signed integer is monotonically decreasing:
\begin{equation}
    FP(X)\geq FP(Y) \Leftrightarrow SI(X) \leq SI(Y)
\end{equation}
\end{lemma}
\begin{proof}
    When the sign bit of both numbers is $1$, the term $(-1)^s$ in \Cref{fpdef} evaluates to $-1$. Then, the interpretation of the \fpname~number is exactly the same as in \Cref{fpmagndef} with a constant factor of $-1$. Therefore, we can write $-1\cdot FP(X)=|FP(X)|$ and $-1\cdot FP(Y)=|FP(Y)|$. Starting from \Cref{lemma:monotonic}, we derive\\
    \begin{equation}
        -1\cdot FP(X) > -1\cdot FP(Y) \Leftrightarrow SI(X) > SI(Y)
    \end{equation}
    which can be transformed into
    \begin{equation}
        FP(X) <  FP(Y) \Leftrightarrow SI(X) > SI(Y)
    \end{equation}
    and further into
    \begin{equation}
        FP(X) \geq  FP(Y) \Leftrightarrow SI(X) \leq SI(Y)
    \end{equation}
\end{proof}

A comparison between two floating point numbers can, in addition to the considered cases, also operate on the mixed case of a positive and negative number. Since the sign bit in \fpname s at the position of the most significant bit also serves in \tkname~as a sign bit, these cases are covered as well.
\begin{lemma}
\label{lemma:mixedmonotonic}
Given two arbitrary bit vectors $X,Y\in \{0,1\}^{k}$ with \textbf{different} sign bits $x_{k-1}\neq y_{k-1}$, the interpretation as \fpname~and as signed integer numbers is strictly monotonically increasing:
\begin{equation}
    FP(X)>FP(Y) \Leftrightarrow SI(X) > SI(Y)
\end{equation}
\end{lemma}
\begin{proof}
Negative numbers in \fpname~are indicated by the sign bit set to $1$, which also indicates a negative number in \tkname. Positive numbers in \fpname~are indicated by the sign bit set to $0$, which also indicates a positive number in \tkname. Hence, numbers are interpreted as negative and positive similarly in \fpname~and in \tkname. If one number is positive and the other is negative, the interpreted absolute value is irrelevant.
\end{proof}

Next, \Cref{lemma:equal} can be used to extend \Cref{lemma:negativemonotonic}:
\begin{lemma}
\label{lemma:negativestrictmonotonic}
Given two arbitrary bit vectors $X,Y\in \{0,1\}^{k}$ with \textbf{negative} sign bit $x_{k-1}=y_{k-1}=1$, the interpretation as \fpname~and as signed integer numbers is strictly monotonically decreasing:
\begin{equation}
    FP(X)> FP(Y) \Leftrightarrow SI(X) < SI(Y)
\end{equation}
\end{lemma}
\begin{proof}
    From \Cref{lemma:equal} we know that the interpretation as \tkname~can only be the same if and only if the interpretation in \fpname~is the same. Thus, $FP(X)=FP(Y) \Leftarrow SI(X)\neq SI(Y)$ or $FP(X)\neq FP(Y) \Rightarrow SI(X)= SI(Y)$ cannot happen. Thus, we can exclude the equality cases from \Cref{lemma:negativemonotonic}.
\end{proof}

\subsection{Design of the \mdname~Operator}

Leveraging the previous lemmata, we construct an evaluation of the $\geq$ relation for \fpname~numbers which only evaluates the $\geq$ and the $<$ (which is the logic negation of $\geq$) relation of \tkname~signed integer numbers.

\begin{corollary}
\label{corolarry:op}
Given two arbitrary bit vectors $X,Y\in \{0,1\}^{k}$, we can compute the $\geq$ relation between the \fpname~interpretation of these bit vectors, using only \tkname~signed integer arithmetic when distinguishing two cases:
\begin{align}
      &  FP(X)\geq FP(Y)\nonumber\\ \Leftrightarrow\nonumber\\
    &    \left\{ \begin{array}{ll}
         SI(X) < SI(Y) & \mbox{ if }FP(X) < 0 \land FP(Y) < 0 \\
         & \land FP(X) \neq FP(Y)\\
         SI(X) \geq SI(Y) &  \mbox{ otherwise}
    \end{array} \right.
\end{align}
\end{corollary}
\begin{proof}
    The first case (both numbers negative) is discussed in \Cref{lemma:negativestrictmonotonic}. It should be noted that this only covers the case that both numbers are negative, but not equal. For the case that the numbers are equal, either positive or negative, \Cref{lemma:equal} shows that the second case holds. Also for the case that both numbers are positive, but not equal, \Cref{lemma:positivemonotonic} shows that the second case holds. For the case that only one number is positive, \Cref{lemma:mixedmonotonic} shows that also the second case holds. Since the second case consists of the latter three cases, all cases are covered. It should be also noted that the condition, whether the first or second case is needed, also can be evaluated on the signed integer representation, according to \Cref{lemma:mixedmonotonic}, \Cref{lemma:positivemonotonic} and \Cref{lemma:equal}. The evaluation of the first and the second number is negative could also be done independent of the format interpretation by only extracting the sign bits $x_{k-1}$ and $y_{k-1}$.
\end{proof}

\begin{theorem}
\label{theorem:intoperator}
Given two arbitrary bit vectors $X,Y\in \{0,1\}^{k}$, we can compute the $\geq$ relation between the \fpname~interpretation of these bit vectors, using only \tkname~signed integer arithmetic, with the following operation:
\begin{align}
      &FP(X)\geq FP(Y)\nonumber \\
      \Leftrightarrow\nonumber \\  
     &\left(SI(X) \geq SI(Y)\right) \oplus\nonumber \\
  & \left((SI(X) < 0 \land SI(Y) < 0 \land SI(X) \neq SI(Y)\right)   
\end{align}
\end{theorem}
Here, we use the XOR function $\oplus$ to achieve negation in case the second input is $true$.
Let $u=(SI(X) \geq SI(Y))$ and $v=\left((SI(X) < 0 \land SI(Y) < 0 \land SI(X) \neq SI(Y)\right)$. Applying XOR to the value $u$ while the second input $v$ is false, evaluates to the identity function ($u\oplus false=u$), applying XOR to the value $u$ while the second input $v$ is true, evaluates to the negation ($u\oplus true=\neg u$).

\begin{proof}

    Since $\neg (SI(X) \geq SI(Y))$  is $SI(X) < SI(Y)$, we know from \Cref{corolarry:op} that we only need to compute $u$ and negate the result when the first case applies. 
    Hence, we evaluate the condition for the first or second case $v$ based on signed integer arithmetic, which delivers true when the condition holds and false when the condition does not hold. In order to achieve negation in case the condition holds, we apply the exclusive or (XOR) function $\oplus$ here. 
\end{proof}

\noindent\textbf{Towards efficient computation:} In the former part of this section, we present a method to perform floating point comparisons by only using signed integer arithmetic and logic operations. In many CPU instructions sets (including X86 and ARMv8), there is no dedicated operation to compute the $<$ or $\neq$ relation. Instead, a comparison instruction needs to be called and a subsequent conditional set or even a conditional branch is required. Hence, the method from \Cref{theorem:intoperator} would require in total four comparisons and conditional set or branch instructions. Depending on the CPU architecture, it may be more efficient to check only if $SI(X) < 0$ and exchange and invert $X$ and $Y$:
\begin{theorem}
\label{theorem:negativeknown}
Given two arbitrary bit vectors $X,Y\in \{0,1\}^{k}$ where the positiveness of $FP(X)$ (equivalently $SI(X)$) is known a priori, the $\geq$ relation can be computed between the \fpname~interpretation of these bit vectors, using only \tkname~signed integer arithmetic:
\begin{align}
      &  FP(X)\geq FP(Y)\nonumber\\ \Leftrightarrow\nonumber\\
    &    \left\{ \begin{array}{ll}
         -1\cdot SI(Y) \geq -1\cdot SI(X) & \mbox{ if }SI(X) < 0 \\
         SI(X) \geq SI(Y) &  \mbox{ otherwise}
    \end{array} \right.
\end{align}
\end{theorem}
\begin{proof}
    Following \Cref{corolarry:op}, the second case is needed when $X$ is positive. If $X$ is negative, the comparison $SI(X) \geq SI(Y)$ can be directly transformed to $-1 \cdot SI(X) \leq -1 \cdot SI(Y) \Leftrightarrow -1\cdot SI(Y) \geq -1\cdot SI(X)$, which is then a comparison with  at least one positive operand, thus the second case from \Cref{corolarry:op} applies again.
\end{proof}

Please note that in \Cref{theorem:negativeknown} always one operand is ensured to be positive for the comparison. Hence, the equivalence $FP(X) \geq FP(Y) \Leftrightarrow SI(X) \geq SI(Y)$ or $FP(X) \geq FP(Y) \Leftrightarrow -1\cdot SI(Y) \geq -1\cdot SI(X)$  holds. This also implies that all other relations ($\leq$, $>$, $<$) hold in the same manner. Especially for integrating \mdname~into program code, this allows the usage of arbitrary comparison constructs.

% \clearpage
\section{Efficient Implementation of FLInts in Low Level Random Forest Execution}
\label{sec_dt}
In this section, we introduce \mdname, an operator to compute the $\geq$ relation of \fpname~numbers by only using \tkname~signed integer arithmetic and logic operations. In the following, we show how \mdnames~can be efficiently utilized in random forest execution and can omit the need of floating point arithmetic entirely. 

\subsection{Design Overview}
We consider a random forest to consist of multiple decision trees. We consider every decision tree to consist of a set of nodes $N=\{n_0, n_1, ..., n_m\}$ where $n_0$ is the root node. Every node is associated with a feature index $FI(n_x)$, a split value $SP(n_x)$, a left child pointer $LC(n_x)$, a right child pointer $RC(n_x)$ and a prediction value $PR(n_x)$. For inner nodes, the prediction value is not needed, for leaf nodes the left and right child pointers are not needed. When executing the tree, a feature vector $F=(f_0, f_1, ..., f_n)$ serves as the input and the ultimate goal is to find the corresponding prediction, associated with the given feature vector. Therefore, the inference begins at the root node and visits a sequence of nodes $n_{x_0}, n_{x_1}, ..., n_{x_o}$, such that
\[n_{x_{i+1}}=\left\{\begin{array}{ll}LC(n_{x_i})& \mbox{if }F(FI(n_{x_i})) \leq SP(n_{x_i})\\RC(n_{x_i})&\mbox{otherwise}\end{array}\right.\] until a leaf node is reached. The prediction of this leaf node is then returned as a result. 
Please note that previously in this paper, we consider the $\geq$ relation only. However, by exchanging the first and the second operand, also the $\leq$ relation is enabled. Furthermore, by negation also the $>$ and the $<$ relation is enabled, thus we use different comparison operations in the following.
The datatype of the feature vector is defined by the data source. The split values are derived during the training of the decision tree from a training data set, which should come from the same source as the input data. Hence, the datatype of the split value follows the datatype of the feature vector. Whenever the feature vector consists of floating point numbers, the decision tree has to perform floating point comparison during inference.

After training, the tree nodes can be stored in an arbitrary form together with their associated values. For efficient execution, however, a decision tree can be implemented into dedicated source code, compiled and executed. For implementations of decision trees, two distinct methods have been identified by \cite{Asadi/etal/2014}: 1) \emph{native trees} where nodes become an array like data structure and a narrow loop reads out the node values and maintains an index of the current node and 2) \emph{if-else trees}, where nodes become nested if-else blocks. As the branch condition, the comparison between the feature value and the split value is taken, the entire further code for the left subtree is then placed into the if block, the entire code for the right subtree is placed into the else block. 

While it is not obvious which one is the superior implementation, if-else trees are shown to be more efficient in several empirical scenarios \cite{chen2022efficient}. Furthermore, cache-aware implementations exist, namely \textbf{arch-forest}, where the if-and-else branches are swapped according to the empirical branching probability from the training data set in order to achieve memory locality. We take the arch-forest framework of this implementation of if-else trees as a basis\footnote{\url{https://github.com/tudo-ls8/arch-forest}} and integrate \mdname~into the code generation process. The modified implementation is publicly available\footnote{\url{https://github.com/tu-dortmund-ls12-rt/arch-forest/tree/flintcomparison}}. This framework is modular and can be extended by additional code generators, which transform random forest models directly into machine code. The training of the trees is based on scikit-learn \cite{JMLR:v12:pedregosa11a}. It should be noted that \mdnames~can also be integrated to native tree implementations in C without further issues.

\subsection{Implementation in C}
\label{sec:C-impl}
As the first approach, we extend the arch-forest framework with a code generator to implement \mdnames~directly in the C code realization of if-else trees. After the implementation, an if-else tree consists of a bunch of nested if-else blocks, illustrated in \Cref{ifelsestandardc}.
In a first step, we reinterpret the \emph{pX} array as an array of signed integer numbers and load the corresponding element. Next, we place the \fpexname~encoding of the comparison value as an immediate constant and also interpret this as a signed integer. Thanks to \Cref{theorem:negativeknown}, the example in \Cref{ifelsestandardc} is equivalent to that in \Cref{ifelseflintc}.
\begin{lstlisting}[language=C, caption=Standard if-else tree in C, captionpos=b, numbers=left, basicstyle=\footnotesize, label=ifelsestandardc, xleftmargin=2em, float=t, belowskip=-.5cm]
if(pX[3] <= (float) 10.074347){
    if(pX[83] <= (float) 11.974715){
        if(pX[24] <= (float) 10430.507324){
            ...
\end{lstlisting}
\begin{lstlisting}[language=C, caption=\mdname~if-else tree in C, captionpos=b, numbers=left, basicstyle=\footnotesize, label=ifelseflintc, xleftmargin=2em, float=t, belowskip=-.5cm]
if((*(((int *)(pX))+3))<=((int)(0x41213087))){
 if((*(((int *)(pX))+83))<=((int)(0x413f986e))){
  if((*(((int *)(pX))+24))<=((int)(0x4622fa08))){
   ...
\end{lstlisting}
Since the split value of a node is a constant during the implementation time, we can resolve the condition for \Cref{theorem:negativeknown} already during the code generation and can exclude the case of comparing two negative \fpname~numbers. For positive split values, the code is generated as in the previous example. For negative split values, we multiply both numbers with $-1$ (flip the sign bit) and inverse the comparison, as illustrated in \Cref{ifelsestandardcneg} and \Cref{ifelseflintcneg}.
\begin{lstlisting}[language=C, caption=Standard if-else tree in C (negative split value), captionpos=b, numbers=left, basicstyle=\footnotesize, label=ifelsestandardcneg, xleftmargin=2em, float=t, belowskip=-.6cm]
if(pX[125] <= (float) -2.935417){
\end{lstlisting}
\begin{lstlisting}[language=C, caption=\mdname~if-else tree in C (negative split value), captionpos=b, numbers=left, basicstyle=\footnotesize, label=ifelseflintcneg, xleftmargin=2em, float=t, belowskip=-.5cm]
if(((int)(0x403bddde))<=(*(((int *)(pX))+125)^
                                (0b1 << 31))){
\end{lstlisting}
This ensures that we always compare either two positive numbers or at least one positive and one negative number. Thus, no further logic operations are required for any case distinction. \mdnames~have one semantic difference from \fpexname, i.e. that \mdname~assumes $-0<+0$ and \fpexname~assumes $-0=+0$. We include that into our implementation by rewriting a split value of $-0$ to $+0$ during the code generation. Since we only need to perform $\leq$ comparisons, the different assumptions of \mdname~and \fpexname~have no impact, as shown in Theorem~\ref{theorem:negativeknown}.
% \kuan{leq or geq??}

\subsection{Other Implementations}
In fact, the \mdname~is not limited to C. Any language that allows a reinterpretation of the incoming floating point data to a signed integer representation is applicable to realize the concept, with various overheads depending on the adopted language.
To explicitly eliminate this kind of overhead, one possibility is to implement \mdname~directly in machine assembly code, but under limited portability and applicability.

\begin{lstlisting}[language={[x86masm]Assembler}, caption=\mdname~assembly implementation (ARMv8), captionpos=b, morekeywords={ldrsw, movz, movk, fmov, fcmp, b.gt}, numbers=left, basicstyle=\footnotesize, label=flintasm, xleftmargin=2em, float=t, belowskip=-.6cm]
ldrsw x1, [%1, 12];
movz x2, #0x3087;
movk x2, #0x4121, lsl 16;
cmp w1, w2;
b.gt __rtitt_lab_0_0;
   ldrsw x1, [%1, 332]
   movz x2, #0x986e;
   movk x2, #0x413f, lsl 16;
   cmp w1, w2;
   b.gt __rtitt_lab_1_0;
      ldrsw x1, [%1, 96];
      movz x2, #0xfa08;
      movk x2, #0x4622, lsl 16;
      cmp w1, w2;
      b.gt __rtitt_lab_2_0;
\end{lstlisting}

An illustrative realization can be found in \Cref{flintasm}, which is derived from our code generator, i.e., an extended arch-forest directly generates standard if-else trees as X86 or ARMv8 assembly code. The pointer to the feature vector is passed as $\%1$ by using inline assembly. Then, in each block the signed word is directly loaded from that address plus the feature offset. Afterwards, the split value is encoded as a constant immediate and loaded to another register, which is then used together with the previously loaded register for comparison and conditional branching. When it comes to negative split values, we also inverse the loaded feature value by using an additional \texttt{eor} (exclusive or) instruction to flip the sign bit. Our code generator supports this kind of direct loading for ARMv8 and X86 processors. Furthermore, single precision (\emph{float}) and double precision (\emph{double}) datatypes are supported.

% \clearpage
\section{Evaluation}
\label{sec_eval}

In the former part of this paper, we discuss a method how to eliminate floating point operations entirely from random forest inference, while not changing the result of the model at all. Although there may be unavoidable motivations to eliminate the use of floating points from a system (e.g. no presence of a hardware floating point unit or high energy consumption of the floating point unit), we study a more general motivation in the following: the reduction of \emph{execution time}. The usage of floating point units can lead to higher execution time for various reasons. On the one hand, floating point operations can simply consume more time than equivalent integer operations. On the other hand, the usage of floating points introduces a certain overhead of usage of dedicated floating point registers and value conversion, which can lead to additional machine instructions. To comprehensively study the impact on the allover performance in terms of execution time for random forests of using \mdname, we conducted experiments on multiple data sets, machine classes and CPU architectures in the following.

\subsection{Evaluation Setup}
We utilized scikit-learn to train multiple random forest configurations on a subset of data sets from the UCI machine learning repository~\cite{Dua:2019}: The \textit{EEG Eye State Data Set} (eye), the \textit{Gas Sensor Array Drift Data Set} (gas), the \textit{MAGIC Gamma Telescope Data Set} (magic), the \textit{Sensorless Drive Diagnosis Data Set} (sensorless) and the \textit{Wine Quality Data Set} (wine). All these data sets contain floating point values, thus scikit-learn inherently created floating point split values for the trained random forests and decision trees. 

For every data set, we trained random forests with $\{1, 5, 10, 15, 20, 30, 50, 80, 100\}$ trees. 
For every random forest size, we limited the maximal depth of all trees to $\{1, 5, 10, 15, 20, 30, 50\}$ layers. Please note that this was only a maximal depth, the training may thus lead to smaller trees, which was not under control. We furthermore did not perform any tuning of hyper parameters and utilize scikit-learn in the standard configuration, since the optimal creation of random forests is out of the scope of this paper. Consequently, we split our data sets into $75\%$ training data and $25\%$ test data and measured the execution time of the random forests only on the formerly unseen test data.

To evaluate the impact of the omission of the use of floating point units on the execution time, we execute the random forests on X86 and ARMv8 systems. For each architecture, we consider a server class and a desktop class system. The machine details can be found in \Cref{tab:machine}. All systems run Linux without any underlying hypervisor or simulation system.
\begin{table*}
    \centering
    \caption{Machine details for evaluation}
    \vspace{-.2cm}
    \begin{tabular}{l|llll}
         \textbf{Machine}&\textbf{System}&\textbf{CPU}&\textbf{RAM}&\textbf{Linux kernel}\\
         \hline
         \textbf{X86 Server}&Gigabyte R182-Z92-00&2x AMD EPYC 7742&256GB DDR4&5.10.0 x86\_64\\
         \textbf{X86 Desktop}&Dell OptiPlex 5090&Intel Core i7-10700&64GB DDR4&5.10.106 x86\_64\\
         \textbf{ARMv8 Server}&Gigabyte R181-T9&2x Cavium ThunderX2 99xx&256GB DDR4&5.4.0 aarch64\\
         \textbf{ARMv8 Desktop}&Apple Mac Mini&Apple Silicon M1&16GB DDR4&5.17.0 aarch64\\
    \end{tabular}
    \vspace{-.4cm}
    \label{tab:machine}
\end{table*}

To compare the achievement in terms of execution time reduction, we consider multiple implementations for every random forest, including the state-of-the-art~\cite{chen2022efficient}:
\begin{enumerate}
    \item A standard if-else tree, where tree nodes are straightforward translated into nested if-else blocks and normal floating point numbers are used
    \item A cache-aware if-else tree implementation\cite{chen2022efficient} (which is an extension of \cite{8594826}), called \kuanmdnameabbr~(\kuanmdname) in the following, where if-else blocks are swapped and jumps between nested if-else blocks are introduced in order to optimize the cache efficiency
    \item The C implementation of the standard if-else with \mdname
    \item An Implementation of \kuanmdnameabbr~with \mdname~integrated
\end{enumerate}
For the latter three implementations, we compute the normalized execution time to the standard implementation, by which we derive a fraction of the execution time of the naive version, which indicates the gained improvement. We further group all configurations with the same maximal tree depth together and present them by their mean normalized execution time and the corresponding variance across all data-sets and number of trees within the ensemble. 

Please note that the integration of \mdname~into \kuanmdnameabbr~is straightforward and does not consider any changes of the original \kuanmdnameabbr~implementation. This algorithm, however, includes assumptions and considerations about the available caches and the usage of cache by decision tree nodes. Due to the implementation of \mdname, several of these assumptions may be violated and need to be re-evaluated, which is out of the scope of this paper. 
In addition, \mdname s are integrated into \kuanmdnameabbr~only as the C-based implementation, since the integration of the assembly version requires dedicated effort to generally let the \kuanmdnameabbr~method directly produce assembly code. Thus, explicitly rewriting \kuanmdnameabbr~into an assembly based generator and accordingly integrate \mdname~could lead to further performance impacts.

\subsection{Evaluation Results}
\begin{figure*}[h!]
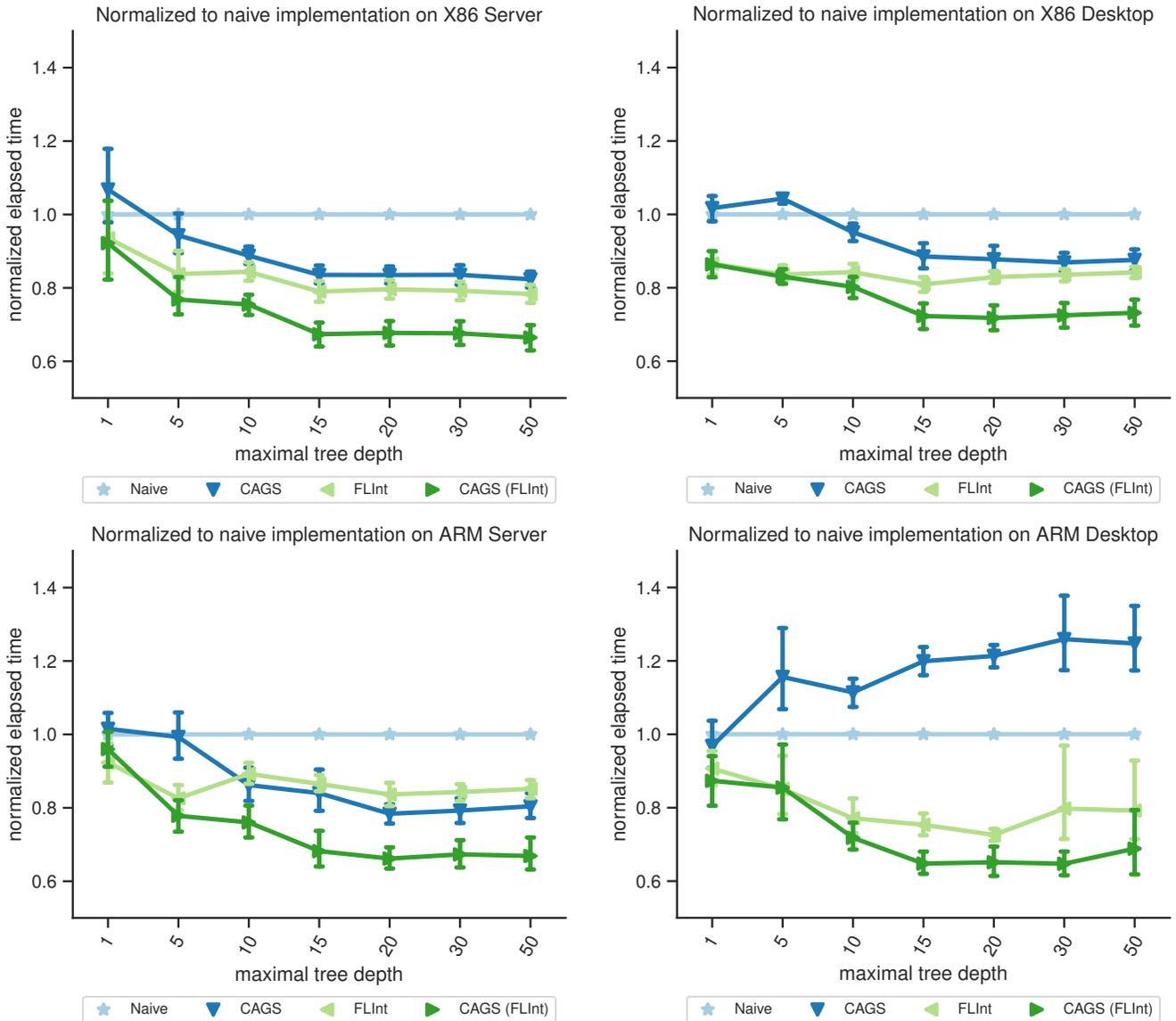

    \begin{minipage}{.49\textwidth}
        \scalebox{.7}{
            \input{plots/results_X86_Server_StandardIfTree-Depth.pgf}
        }
    \end{minipage}
    \begin{minipage}{.49\textwidth}
        \scalebox{.7}{
            \input{plots/results_X86_Desktop_StandardIfTree-Depth.pgf}
        }
    \end{minipage}
    
    \begin{minipage}{.49\textwidth}
        \scalebox{.7}{
            \input{plots/results_ARM_Server_StandardIfTree-Depth.pgf}
        }
    \end{minipage}
    \begin{minipage}{.49\textwidth}
        \scalebox{.7}{
            \input{plots/results_ARM_Desktop_StandardIfTree-Depth.pgf}
        }
    \end{minipage}
    \caption{Normalized execution time for increasing maximal tree depth}
    \vspace{-.5cm}
    \label{fig:results}
\end{figure*}
We illustrate the average (geometric mean) normalized execution time across all data sets and ensemble sizes for a specific maximal depth of the single trees in \Cref{fig:results} for all considered test systems. The naive standard implementation is illustrated in light blue with star tick marks as the baseline. The cache-aware implementation (\kuanmdnameabbr) for if-else trees from Chen et~al.~\cite{chen2022efficient} is illustrated in dark blue with down arrow tick marks. The \mdname~results for standard trees are then depicted by a light green line with left arrow tick marks. The \kuanmdnameabbr~implementation with integrated \mdname s is depicted by a dark green line and right arrow tick marks. While the x-axis depicts a growing maximal depth of trees, the y-axis presents the fraction of execution time from the naive baseline version. A value of e.g. $0.75$ thus indicates $25\%$ improvement in execution time. A value larger than $1$ consequently describes an increase of the execution time. Each point in the plots is also associated with the computed variance across all data sets and ensemble sizes.
\newcommand{\badcol}{red!80!black}
\newcommand{\goodcol}{green!60!black}
\begin{table}
    \centering
    \caption{Average (geometric mean) normalized execution time: ($D\geq20$): Average of ensembles with a maximal tree depth of more than 20}
    \begin{tabular}{l|rrrr}
    &\textbf{X86 S}&\textbf{X86 D}&\textbf{ARMv8 S}&\textbf{ARMv8 D}\\
    \hline
    \textbf{\kuanmdnameabbr}&\textcolor{\goodcol}{0.88$\times$}&\textcolor{\goodcol}{0.92$\times$}&\textcolor{\goodcol}{0.85$\times$}&\textcolor{\badcol}{1.14$\times$}\\
    \textbf{\kuanmdnameabbr~($D\geq20$)}&\textcolor{\goodcol}{0.83$\times$}&\textcolor{\goodcol}{0.87$\times$}&\textcolor{\goodcol}{0.79$\times$}&\textcolor{\badcol}{1.22$\times$}\\
    \textbf{FLInt}&\textcolor{\goodcol}{0.81$\times$}&\textcolor{\goodcol}{0.83$\times$}&\textcolor{\goodcol}{0.85$\times$}&\textcolor{\goodcol}{0.77$\times$}\\
    \textbf{FLInt ($D\geq20$)}&\textcolor{\goodcol}{0.79$\times$}&\textcolor{\goodcol}{0.83$\times$}&\textcolor{\goodcol}{0.84$\times$}&\textcolor{\goodcol}{0.74$\times$}\\
    \textbf{\kuanmdnameabbr~(\mdname)}&\textcolor{\goodcol}{0.71$\times$}&\textcolor{\goodcol}{0.76$\times$}&\textcolor{\goodcol}{0.72$\times$}&\textcolor{\goodcol}{0.70$\times$}\\
    \textbf{\kuanmdnameabbr~(\mdname)~($D\geq20$)}&\textcolor{\goodcol}{0.66$\times$}&\textcolor{\goodcol}{0.72$\times$}&\textcolor{\goodcol}{0.66$\times$}&\textcolor{\goodcol}{0.64$\times$}\\
    \end{tabular}
    \vspace{-.6cm}
    \label{tab:resultnumbers}
\end{table}
In addition to the graphical illustration, we also provide the average (geometric mean) normalized execution time in \Cref{tab:resultnumbers}. We compute the average over two sets: 1) all tree configurations for all benchmarks for one implementation and 2) all tree configurations where the maximal depth is limited to more than 20 for all benchmarks for one implementation.

\noindent\textbf{Results in general:} From the presented results, several observations can be made. First, it can be observed for almost all systems and configurations that the gained execution time improvement varies much for small trees and reaches a more constant value for deeper trees. For small trees, a normalized short time is spent for every feature vector for traversing the tree, which imposes a higher contribution of overheads (e.g. creating of data structures and function calls). For higher maximal depths of the trees, single trees do not reach the maximal depth at a certain point (when the data set requires no further splitting to gain accuracy), hence trees can have a similar shape for high maximal depths. Second, it can be observed that the \mdname~implementation improves the execution time for almost all evaluated cases for the standard tree, as well as for the \kuanmdnameabbr~implementation.

\noindent\textbf{Integration into \kuanmdnameabbr:} In order to asses the range of improvement in terms of execution time with other state-of-the-art optimization approaches for decision trees and evaluate how \mdname~can work together with such optimizations, we compare \mdname~to \kuanmdnameabbr~from Chen et~al.~\cite{chen2022efficient}.
For all systems, except the ARM 
server system, \mdname~on its own achieves a similar or larger improvement as \kuanmdnameabbr~does. For smaller trees, the improvement is even consequently larger.
Basically, \mdname s can be also integrated into \kuanmdnameabbr~directly. The approach however, explicitly considers instruction and data caches for the implementation. Since floating point constants are usually loaded from data memory, but encoded as immediate values in \mdname~and thus are loaded from instruction memory, the optimization algorithm has to be redesigned to properly work together with \mdname s. Ignoring this for a moment and investigating the results of the straightforward integration of \mdname~into \kuanmdnameabbr, it can be seen that the performance is improved significantly in almost all cases. Furthermore, the improvement seems to be almost constant over different sized trees for almost all systems. This suggests the conclusion that \mdname~is an orthogonal optimization to \kuanmdnameabbr~and optimizes another performance bottleneck, working well together with \kuanmdnameabbr. 

\noindent\textbf{Direct Assembly Implementation:} We motivate the direct assembly implementation of \mdname~by eliminating language related overheads for the reinterpretation of floating point values. Therefore, we also compare the direct assembly based implementation of \mdname~with the C-based implementation of standard trees. It should be noted that the assembly based implementation could also be combined with \kuanmdnameabbr~but requires rewriting of the entire \kuanmdnameabbr~algorithm to directly produce assembly code. Since this imposes methodological changes in the algorithm itself, which open another design space, it is out the scope of this work.
\begin{figure}[]
\centering
        \scalebox{.7}{
            \input{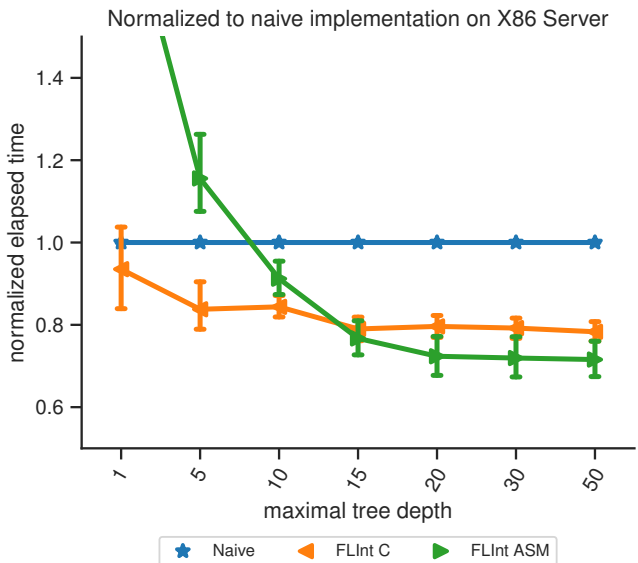}
        }
    \vspace{-.6cm}
    \caption{Normalized execution time for the assembly and C implementation}
    \vspace{-.2cm}
    \label{fig:results_asm}
\end{figure}
\Cref{fig:results_asm} highlights the normalized execution time for the direct assembly implementation (orange line with left arrow ticks) in relation to the C-based implementation (green line with right arrow ticks) for the X86 server system. It can be observed that although the assembly version performs worse for small tree sizes due to the missing compiler optimization, it can outperform the C-based implementation for larger trees. This is also consistent with the other tested systems (\Cref{tab:resultnumbersasm}).
\begin{table}[t]
    \centering
    \caption{Average normalized execution time for the assembly implementation:}
    \begin{tabular}{l|rrrr}
    &\textbf{X86 S}&\textbf{X86 D}&\textbf{ARMv8 S}&\textbf{ARMv8 D}\\
    \hline
    \textbf{FLInt ASM}&\textcolor{\goodcol}{0.89$\times$}&\textcolor{\goodcol}{0.95$\times$}&\textcolor{\goodcol}{0.83$\times$}&\textcolor{\goodcol}{0.89$\times$}\\
    \textbf{FLInt ASM ($D\geq20$)}&\textcolor{\goodcol}{0.70$\times$}&\textcolor{\goodcol}{0.75$\times$}&\textcolor{\goodcol}{0.69$\times$}&\textcolor{\goodcol}{0.72$\times$}\\
    \end{tabular}
    \label{tab:resultnumbersasm}
    \vspace{-.6cm}
\end{table}
This suggests the conclusion that the assembly implementation could gain higher performance improvements, when integrated for deep trees.

Overall, it can be observed that the integration of \mdname s into random forests can reduce the execution time in comparison to a naive implementation by up to $\approx30\%$. Integrating \mdname s into other existing optimization methods, even reduces the execution time by up to $\approx 35\%$. While it makes sense to 
utilize the C-based implementation for small trees, the assembly based implementation can achieve higher performance gains for deeper trees due to the explicit control over the value loading and interpretation. 

% \clearpage
\vspace{-.1cm}
\section{Conclusion}
\label{sec_conc}
In this paper, we discuss the realization of floating point comparison by only using \tkname~and logic operations. We prove that the resulting \mdname~operator delivers correct results. We further efficiently integrate \mdname s into random forest implementations as if-else trees and resolve the handling of special cases already during the implementation time. With that, we provide an option to execute floating point based random forests without any use of hardware floating point units or software floats. This not only allows the execution of such models on devices without floating point support, but also improves the performance on other devices, by excluding overheads of the floating point operations.

In practice, we provide a generic C-based implementation and a specialized assembly-based implementation for X86 and ARMv8, publicly available: \url{https://github.com/tu-dortmund-ls12-rt/arch-forest/tree/flintcomparison}. Our evaluation on server and desktop class systems shows that our approach can improve the performance in almost all considered cases upon a naive realization. 
Overall, integrating \mdname~into random forests reduces the execution time by up to $\approx 30\%$ and even up to $\approx 35\%$ with additional cache aware optimization in our tested cases.

For future work, the interplay between \mdname~and \kuanmdnameabbr~can be improved. The assembly implementation can be integrated and the assumptions about available cache sizes can be adjusted.
In addition, \mdname s can be integrated into other applications, which heavily rely on floating point comparisons. In order to motivate the usage of \mdname, we plan to publish our code.
\vspace{-.2cm}

\section*{Acknowledgement}
\vspace{-.1cm}
This work has been supported by Deutsche Forschungsgemeinschaft (DFG) within the project OneMemory (project number 405422836), the SFB876 A1 (project number 124020371), and Deutscher Akademischer Austauschdienst (DAAD) within the Programme for Project-Related Personal Exchange (PPP) (project number 57559723).
\vspace{-.1cm}

\bibliographystyle{abbrv}
\bibliography{sources}

\begin{thebibliography}{10}

\bibitem{30711}
{IEEE} standard for binary floating-point arithmetic.
\newblock {\em ANSI/IEEE Std 754-1985}, pages 1--20, 1985.

\bibitem{Asadi/etal/2014}
N.~Asadi, J.~Lin, and A.~P. de~Vries.
\newblock Runtime optimizations for tree-based machine learning models.
\newblock {\em IEEE Transactions on Knowledge and Data Engineering},
  26(9):2281--2292, Sept 2014.

\bibitem{blinn1997floating}
J.~F. Blinn.
\newblock Floating-point tricks.
\newblock {\em IEEE Computer Graphics and Applications}, 17(4):80--84, 1997.

\bibitem{bramley}
J.~Bramley.
\newblock Condition codes 4: Floating-point comparisons using {VFP}.
\newblock
  \url{https://community.arm.com/arm-community-blogs/b/architectures-and-processors-blog/posts/condition-codes-4-floating-point-comparisons-using-vfp}.

\bibitem{8594826}
S.~Buschj\"ager, K.-H. Chen, J.-J. Chen, and K.~Morik.
\newblock Realization of random forest for real-time evaluation through tree
  framing.
\newblock In {\em 2018 IEEE International Conference on Data Mining (ICDM)},
  2018.

\bibitem{chen2022efficient}
K.-H. Chen, C.~Su, C.~Hakert, S.~Buschj{\"a}ger, C.-L. Lee, J.-K. Lee,
  K.~Morik, and J.-J. Chen.
\newblock Efficient realization of decision trees for real-time inference.
\newblock {\em Transactions on Embedded Computing Systems}, 2022.

\bibitem{cowlishaw2003decimal}
M.~F. Cowlishaw.
\newblock Decimal floating-point: Algorism for computers.
\newblock In {\em Proceedings 2003 16th IEEE Symposium on Computer Arithmetic},
  pages 104--111. IEEE, 2003.

\bibitem{demmel1995correctness}
J.~W. Demmel, I.~Dhillon, and H.~Ren.
\newblock On the correctness of some bisection-like parallel eigenvalue
  algorithms in floating point arithmetic.
\newblock {\em Electronic Trans. Num. Anal}, 3:116--140, 1995.

\bibitem{Dua:2019}
D.~Dua and C.~Graff.
\newblock Uci machine learning repository, 2017.

\bibitem{kim/etal/2010}
C.~Kim, J.~Chhugani, N.~Satish, E.~Sedlar, A.~Nguyen, T.~Kaldewey, V.~Lee,
  S.~Brandt, and P.~Dubey.
\newblock {FAST}: {F}ast architecture sensitive tree search on modern {CPUs}
  and {GPUs}.
\newblock In {\em Proceedings of the International Conference on Management of
  data}. ACM, 2010.

\bibitem{10.1145/2911451.2914758}
C.~Lucchese, F.~M. Nardini, S.~Orlando, R.~Perego, N.~Tonellotto, and
  R.~Venturini.
\newblock Exploiting cpu simd extensions to speed-up document scoring with tree
  ensembles.
\newblock In {\em Proceedings of the 39th International ACM SIGIR Conference on
  Research and Development in Information Retrieval}, pages 833–--836, 2016.

\bibitem{Lucchese2016}
C.~Lucchese, R.~Perego, F.~M. Nardini, N.~Tonellotto, S.~Orlando, and
  R.~Venturini.
\newblock {Exploiting CPU SIMD extensions to speed-up document scoring with
  tree ensembles}.
\newblock In {\em Proceedings of the International Conference on Research and
  Development in Information Retrieval}, 2016.

\bibitem{10.5555/3488766.3488817}
S.~Nakandala, K.~Saur, G.-I. Yu, K.~Karanasos, C.~Curino, M.~Weimer, and
  M.~Interlandi.
\newblock A tensor compiler for unified machine learning prediction serving.
\newblock In {\em Proceedings of the 14th USENIX Conference on Operating
  Systems Design and Implementation}, pages 899--917, 2020.

\bibitem{10.5555/3153875}
D.~A. Patterson and J.~L. Hennessy.
\newblock {\em Computer Organization and Design RISC-V Edition: The Hardware
  Software Interface}.
\newblock Morgan Kaufmann Publishers Inc., San Francisco, CA, USA, 1st edition,
  2017.

\bibitem{JMLR:v12:pedregosa11a}
F.~Pedregosa, G.~Varoquaux, A.~Gramfort, V.~Michel, B.~Thirion, O.~Grisel,
  M.~Blondel, P.~Prettenhofer, R.~Weiss, V.~Dubourg, J.~Vanderplas, A.~Passos,
  D.~Cournapeau, M.~Brucher, M.~Perrot, and {{\'E}}douard Duchesnay.
\newblock Scikit-learn: Machine learning in python.
\newblock {\em Journal of Machine Learning Research}, 12(85):2825--2830, 2011.

\end{thebibliography}

\end{document}